%% file: JPCSLaTeXGuidelines.tex
\documentclass[a4paper]{jpconf}
\usepackage{graphicx,amsfonts}
\usepackage{amsthm,amsmath,bm}
\newtheorem{theorem}{Theorem}

\newtheorem{definition}{Definition}

\newtheorem{proposition}[theorem]{Proposition}

 \usepackage{multirow} 
\usepackage{comment}
\usepackage{setspace}
\usepackage{booktabs}
\begin{document}
\title{Projection-Based Correction for Enhancing Deep Inverse Networks}

\author{Jorge Bacca}

\address{Department of Computer Science, Universidad Industrial de Santander, Bucaramanga, Colombia}

\ead{jbacquin@uis.edu.co}

\begin{abstract}
Deep learning-based models have demonstrated remarkable success in solving ill-posed inverse problems; however, many fail to strictly adhere to the physical constraints imposed by the measurement process. In this work, we introduce a projection-based correction method to enhance the inference of deep inverse networks by ensuring consistency with the forward model. Specifically, given an initial estimate from a learned reconstruction network, we apply a projection step that constrains the solution to lie within the valid solution space of the inverse problem. We theoretically demonstrate that if the recovery model is a ``well-trained deep inverse network'', the solution can be decomposed into range-space and null-space components, where the projection-based correction reduces to an identity transformation. Extensive simulations and experiments validate the proposed method, demonstrating improved reconstruction accuracy across diverse inverse problems and deep network architectures.
\end{abstract}

\input{sec/1_intro}

\input{sec/2_related_work}

\input{sec/3_method}

\input{sec/4_Simulation_and_Results}

\input{sec/X_suppl}

\break
\section*{Reference }

\bibliography{JPCSLaTeXGuidelines}

\end{document}

%% file: sec/1_intro.tex
\section{Introduction}
\label{sec:intro}

Linear inverse problems are fundamental in computational imaging, signal processing, and various scientific and engineering applications~\cite{ongie2020deep,bacca2023computational}. Mathematically, an inverse problem involves recovering an unknown signal $\boldsymbol{x} \in \mathbb{R}^n$ from indirect or incomplete measurements $\boldsymbol{y} \in \mathbb{R}^m$, modeled as:
\begin{equation}
    \boldsymbol{y} = \mathbf{A} \boldsymbol{x} + \boldsymbol{n},
\end{equation}
where $\mathbf{A} \in \mathbb{R}^{m \times n}$ represents the linear forward sensing operator, and $\boldsymbol{n}$ denotes additive noise. These problems are often \textit{ill-posed}, meaning that either the solution is not unique (i.e., $\text{rank}(\mathbf{A}) < n$) or small perturbations in $\boldsymbol{y}$ lead to large variations in the estimated reconstruction $\hat{\boldsymbol{x}}$.

Deep learning-based methods have emerged as powerful tools for solving inverse problems. Given a paired dataset 
$\{\boldsymbol{y}_i, \boldsymbol{x}_i\}_{i=1}^N$, deep neural networks (DNNs) aim to learn a mapping function 
$f: \mathbb{R}^m \to \mathbb{R}^n$ that estimates the reconstruction $\hat{\boldsymbol{x}} = f(\boldsymbol{y})$. 
These approaches can be broadly categorized into two main types: \textit{Pure Deep-Learning Models:} These methods rely solely on advanced image-to-image neural network architectures to refine an 
initial estimate, typically computed as $\mathbf{A}^\top\boldsymbol{y}$ or $\mathbf{A}^\dagger\boldsymbol{y}$ 
(where $\dagger$ denotes the Moore-Penrose pseudoinverse). Various architectures have been 
explored, including fully convolutional neural networks such as UNet~\cite{ronneberger2015u}, well-trained denoiser as DnCNN~\cite{zhang2017beyond}, DRUNet~\cite{zhang2021plug}, using transformer-based models as SwinLR~\cite{liang2021swinir} or Restormer~\cite{zamir2022restormer} or diffusion model such DiffUnet~\cite{choi2021ilvr}.
\textit{Physics-Guided and Hybrid Models:} These approaches integrate physical priors into the network architecture to enhance interpretability and generalization. Methods such as Plug-and-Play (PnP)~\cite{chan2016plug}, Regularization by denoising (RED)~\cite{romano2017little}  unrolled optimization networks~\cite{monga2021algorithm}, are based on optimization models, ensuring that reconstructed solutions remain consistent with known measurement processes.

However, in various applications such as computer graphics, biological measurements, image compression, controlled laboratory conditions, and cryptographic encoding, it is essential that the reconstructed signal strictly adheres to the physical constraints imposed by the forward model. Despite this requirement, many deep learning-based inverse models fail to fully satisfy these constraints, leading to reconstructions that may violate fundamental properties of the measurement process. Mathematically, deep networks must learn both the range-space $\mathcal{R}(\mathbf{A}^\top)$ and the null-space $\mathcal{N}(\mathbf{A})$ components of the inverse problem, leading to reconstructions that may not fully respect the measurement process, i.e., 
\begin{equation}
    f(\boldsymbol{y}) =\boldsymbol{x}_{\text{range}} + \boldsymbol{x}_{\text{null}}, \quad \text{where} \quad \mathbf{A} \boldsymbol{x}_{\text{null}} = 0.
\end{equation}
This decomposition highlights a critical issue: the null-space component $\boldsymbol{x}_{\text{null}}$ is unconstrained by the forward model, potentially introducing artifacts or hallucinated features in the reconstruction. Such inconsistencies reduce the reliability of deep inverse models, particularly in safety-critical applications.

To mitigate this issue, we propose a novel \textit{projection-based correction method} that refines the output of deep inverse networks by enforcing strict consistency with the measurement model. Specifically, we introduce two projection strategies based on the noise properties of the measurements. The first is the noise-free scenario, where the physical constraints must be exactly satisfied, i.e., $\boldsymbol{y} = \mathbf{A} \hat{\boldsymbol{x}}$. The second scenario accounts for additive noise, where we incorporate knowledge of the noise covariance to improve robustness. This method serves as a refinement step for any trained network, ensuring consistency with the forward model. Furthermore, we define a \textit{``well-trained deep inverse network''} as one that inherently decomposes its learned representation into range and null-space components for a given dataset. If the reconstruction network $f(\boldsymbol{y})$ satisfies this decomposition, our projection-based correction reduces to an identity transformation, preserving the network's output. Otherwise, the proposed method enhances reconstruction fidelity by eliminating physically inconsistent components. We validate our approach through extensive simulations across multiple inverse problems and deep architectures. Experimental results demonstrate significant improvements in reconstruction accuracy in low-level noise scenarios, confirming the effectiveness of the proposed projection-based correction method.

%% file: sec/2_related_work.tex
\section{Related Work}

Deep inverse models have gained significant attention for solving ill-posed inverse problems by leveraging data-driven approaches to approximate the inverse mapping of complex measurement processes, surpassing traditional iterative optimization algorithms.

\textbf{\textit{Data-Driven Inverse Models.}} Supervised deep learning approaches aim to learn an advanced neural network that maps measurements $\boldsymbol{y}$ to the underlying signal $\boldsymbol{x}$. Consequently, well-known image-to-image networks such as UNet~\cite{ronneberger2015u}, well-trained denoisers like DnCNN~\cite{zhang2017beyond} and DRUNet~\cite{zhang2021plug}, transformer-based architectures such as SwinIR~\cite{liang2021swinir} and Restormer~\cite{zamir2022restormer}, and diffusion models like DiffUnet~\cite{choi2021ilvr} have been successfully applied to refine an initial estimate, typically computed as $\mathbf{A}^\top\boldsymbol{y}$ or $\mathbf{A}^\dagger\boldsymbol{y}$. These models are trained to learn a direct mapping from corrupted measurements to high-quality reconstructions. While effective, they do not inherently satisfy the constraints of the underlying physical system, as the sensing model is only used for the initial estimation. This lack of explicit measurement consistency can lead to hallucinations or artifacts in the reconstructed images.

\textbf{\textit{Physics-Guided and Hybrid Approaches.}} To address the inconsistency issues of purely data-driven approaches, physics-guided deep learning methods have been introduced, integrating known physical models into neural network architectures~\cite{arridge2019solving, chan2016plug, romano2017little, monga2021algorithm, adler2018learned}. These approaches explicitly embed the forward model into the reconstruction process. Hybrid methods, such as deep unfolding networks~\cite{monga2021algorithm} and variational networks~\cite{shah2018solving}, integrate iterative optimization schemes into deep learning frameworks. While these methods improve data consistency, they often struggle with strictly enforcing measurement constraints, particularly in ill-posed inverse problems where significant information is lost in the null space of the forward operator.

\textbf{\textit{Null-Space Learning.}} Null-space learning has emerged as an alternative paradigm that enforces data consistency by ensuring that the learned network operates only in the null space of the forward operator~\cite{chen2020deep, sonderby2016amortised,schwab2019deep,goppel2023data,angermann2023uncertainty,wang2022zero}. Specifically, these methods employ residual architectures that constrain the learning process to focus on the null space component while guaranteeing data consistency through a custom loss function~\cite{schwab2019deep}.  For instance, ~\cite{goppel2023data,sonderby2016amortised} proposed a data-proximal null-space network that enforces data consistency, while~\cite{angermann2023uncertainty} introduced an uncertainty quantification loss function tailored for null-space networks. These \textit{Null-space learning Methods} effectively maintain consistency with the measurement process, but their reliance on costly training iterative networks can lead to suboptimal performance in noisy environments. Although methods have been employed to address noise, such as ~\cite{chen2020deep} that learning an additional residual network for the range-space, these approaches rely on structured networks that must be carefully tuned to integrate with new inversion models. In contrast, our approach presents an alternative to the traditional learning process by introducing a non-iterative projection step for any deep inverse model. This step guarantees data consistency and demonstrates robust performance in the presence of noise, overcoming one of the key limitations of current methods.

\textbf{\textit{Projection-Based Refinements.}} Projection methods have been widely used in classical inverse problems, where an initial estimate is iteratively refined by enforcing constraints~\cite{hanke1997regularizing}. In the context of deep learning, projection-based refinements have been explored in various ways, including iterative gradient refinement using CNNs~\cite{gupta2018cnn}, gradient descent optimization constrained via GANs~\cite{raj2019gan}, proximal projection steps applied to inverse problems~\cite{rick2017one,mardani2018neural} and diffusion model to constrain the solution~\cite{choi2021ilvr,daras2024survey}.  However, to the best of our knowledge, no prior work has introduced a general projection-based correction method that can be seamlessly integrated into any deep inverse network to ensure measurement consistency in a single non-iterative step. Our approach explicitly applies a projection step after the deep network inference, enforcing measurement consistency while preserving learned data priors. This method effectively reduces the null-space component, leading to more stable and accurate reconstructions across diverse inverse problems.

%% file: sec/3_method.tex
\section{Deep Inverse Networks}

A common approach to solving inverse problems is to leverage supervised learning techniques. Given a training dataset consisting of paired observations $(\boldsymbol{x}, \boldsymbol{y})$, a reconstruction network $f$ is trained by minimizing the following objective function
\begin{equation}
    \hat{f}\in \arg\min_f \mathbb{E}_{\boldsymbol{x},\boldsymbol{y}} \| f(\boldsymbol{y}) - \boldsymbol{x} \|_2^2,
    \label{eq:training}
\end{equation}
where the optimal estimator is the minimum mean squared error (MMSE) solution $\hat{f}(\boldsymbol{y})\approx \mathbb{E}
\{\boldsymbol{x}|\boldsymbol{y}\}$. In practice, the expectation is approximated with a finite dataset $\mathcal{X}_o$ with $N$ samples, leading to the empirical loss $\frac{1}{N}\sum_{i=1}^N \| f(\boldsymbol{y}_i) - \boldsymbol{x}_i \|_2^2$. Consequently, the learned function $\hat{f}(\boldsymbol{y})$ is influenced by the structure and diversity of $\mathcal{X}_o$, which means that the model may fail to generalize well to out-of-distribution samples. To formalize the conditions under which a reconstruction network can be considered well-trained, we introduce the following theoretical framework, which will be further explored in the next section.

\begin{definition}[Well-Trained Deep Inverse Network]
\label{defintion_well}
Let $ \hat{f}: \mathcal{Y} \to \mathcal{X} $ be a reconstruction network mapping measurements $ y \in \mathcal{Y} $ to estimates $ f(\boldsymbol{y}) \in \mathcal{X} $, where $ \boldsymbol{y} = \mathbf{A} \boldsymbol{x} $ for some linear operator $ \mathbf{A} : \mathcal{X} \to \mathcal{Y} $ and true signal $ \boldsymbol{x} \in \mathcal{X} $. The network $ \hat{f} $ is considered well-trained over a subset $ \mathcal{X}_0 \subseteq \mathcal{X} $ if, for all $ \boldsymbol{x} \in \mathcal{X}_0 $, the following conditions hold:

\begin{enumerate}
    \item The network output can be decomposed as
   \begin{equation}
   \hat{f}(\boldsymbol{y}) = \mathbf{A}^\dagger \boldsymbol{y} + \mathcal{N}(\boldsymbol{y}),
   \end{equation}
   where $ \mathbf{A}^\dagger $ is the Moore-Penrose pseudoinverse of $ \mathbf{A} $, and $ \mathcal{N}(\boldsymbol{y}) $ lies in the null space of $ \mathbf{A} $, i.e., $ \mathbf{A} \mathcal{N}(\boldsymbol{y}) = \mathbf{0} $.

   \item The term $ \mathcal{N}(\boldsymbol{y}) $ accurately represents the component of $ \boldsymbol{x} $ in the null space of $ \mathbf{A} $, such that
   \begin{equation}
   \mathcal{N}(\boldsymbol{y}) = (\mathbf{I} - \mathbf{A}^\dagger \mathbf{A}) \boldsymbol{x} \quad \text{for all} \quad \boldsymbol{x} \in \mathcal{X}_0.
   \end{equation}
\end{enumerate}
\end{definition}

Under this definition, the network $\hat{f}$ effectively captures both the range and null-space components of the true signal $\boldsymbol{x}$, ensuring accurate and consistent reconstruction. We now demonstrate that if the network is well-trained, then it is a minimizer of Eq.~\eqref{eq:training}.

\begin{proposition} Let $ \hat{f} $ be a reconstruction network satisfying the conditions of a well-trained deep inverse network defined in Definition~\ref{defintion_well}.
Then, the expected reconstruction error over the training dataset satisfies
\begin{equation}
    \mathbb{E}_{\boldsymbol{x},\boldsymbol{y}} \| \hat{f}(\boldsymbol{y}) - \boldsymbol{x} \|_2^2 = 0.
\end{equation}
Thus, $\hat{f}$ is a minimizer of \eqref{eq:training}.
\end{proposition}
\begin{proof}
    The proof is included in Appendix A. 
\end{proof}

\begin{figure*}[!t]
    \centering
    \includegraphics[width=\linewidth]{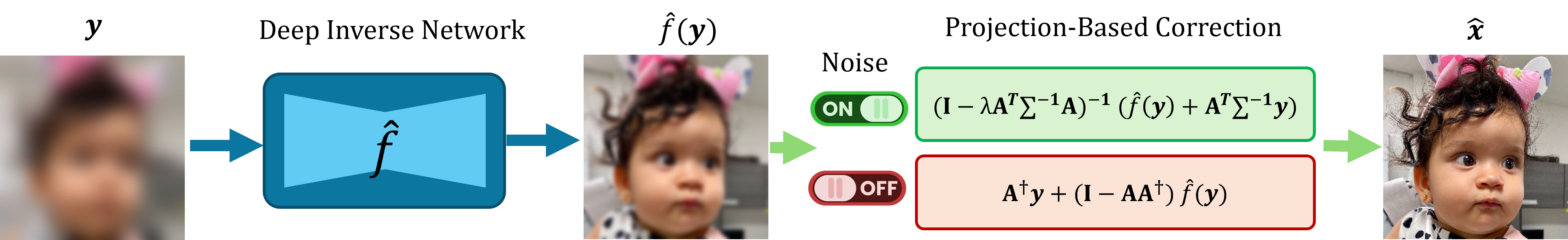}
    \caption{\textbf{Projection-Based Correction Method.} he measurement $\boldsymbol{y}$ is processed through a trained network $\hat{f}$ to obtain an initial estimation $\hat{f}(\boldsymbol{y})$. This estimation is then refined using a fixed projection operator to enhance the reconstruction while ensuring consistency with the measurement model $\hat{\boldsymbol{x}}$.}
    \label{fig:gr2r} \vspace{-1.5em}
\end{figure*}

\section{Projection-Based Correction}

Despite the optimization process described in Eq.~\ref{eq:training}, the estimated reconstruction $\hat{f}(\boldsymbol{y})$ may not fully conform to the measurement model due to limitations in the training procedure. In particular, it is often the case that $\mathbf{A}\hat{f}(\boldsymbol{y}) \neq \boldsymbol{y}$, which implies that the residual error, \( \|\mathbf{A}\hat{f}(\boldsymbol{y}) -\boldsymbol{y}\|_2^2 \) is nonzero. This discrepancy arises because the reconstruction network is trained to minimize reconstruction error rather than strictly enforce consistency with the measurement equation. Such inconsistencies are particularly noticeable in the testing dataset, where the model encounters data it has not seen during training~\footnote{It is important to note that we focus on the noiseless case, where this assumption is typically reasonable. The impact of noise is analyzed in Section ~\ref{section:Noise}.}.

To address this issue and ensure measurement consistency, we seek a corrected solution $\hat{\boldsymbol{x}}$ that satisfies the measurement equation exactly for a given $\boldsymbol{y}$. Specifically, we formulate the following constrained optimization problem:
\begin{equation}
    \min_{\boldsymbol{x}} \|\boldsymbol{x} - \hat{f}(\boldsymbol{y})\|_2^2 \quad \text{subject to} \quad \mathbf{A} \boldsymbol{x} = \boldsymbol{y}.
\end{equation}
This optimization problem guarantees that the final corrected estimate remains as close as possible to the initial reconstruction $\hat{f}(\boldsymbol{y})$ while strictly enforcing the consistency condition $\mathbf{A} \boldsymbol{x} = \boldsymbol{y}$.

The solution to this problem is obtained via a projection-based approach, resulting in the following closed-form expression:
\begin{equation}
    \hat{\boldsymbol{x}} = \mathbf{A}^\dagger \boldsymbol{y} + (\mathbf{I} - \mathbf{A}^\dagger \mathbf{A})\hat{f}(\boldsymbol{y}).
\end{equation}
This formulation decomposes the estimate into two components: (i) $\mathbf{A}^\dagger \boldsymbol{y}$, which represents the least-norm solution satisfying $\mathbf{A}\boldsymbol{x} = \boldsymbol{y}$, and (ii) a correction term that ensures $\boldsymbol{x}$ remains within the feasible subspace defined by the measurement model while preserving the structure of $\hat{f}(\boldsymbol{y})$. 

The following theorem demonstrates that this solution is indeed the optimal solution to the given constrained optimization problem.

\setcounter{theorem}{0}
\begin{theorem}
\cite{sonderby2016amortised}. Let $A \in \mathbb{R}^{m \times n}$ be a linear measurement operator, and let $\boldsymbol{y} \in \mathbb{R}^m$ be the observed measurement. Given an network estimation $\hat{f}(\boldsymbol{y}) \in \mathbb{R}^n$, the solution to the optimization problem
\begin{equation}
    \min_{\boldsymbol{x}} \|\boldsymbol{x} - \hat{f}(\boldsymbol{y})\|_2^2 \quad \text{subject to} \quad \mathbf{A} \boldsymbol{x} = \boldsymbol{y},
\end{equation}
is given by the closed-form expression:
\begin{equation}
    \boldsymbol{x}^* = \mathbf{A}^\dagger \boldsymbol{y} + (\mathbf{I} - \mathbf{A}^\dagger \mathbf{A}) \hat{f}(\boldsymbol{y}),
\end{equation}
where $\mathbf{A}^\dagger \in \mathbb{R}^{n \times m}$ is the Moore-Penrose pseudoinverse of $\mathbf{A}$.
\end{theorem}
\begin{proof}
     The proof is included in Appendix A. 
\end{proof}

The proposed approach can be regarded as a refinement method for the estimated image, where the projection-based correction enforces the measurement constraints. This method is particularly beneficial for addressing imperfect reconstructions from learned models, ensuring that the final estimate adheres to the underlying physical principles of the measurement process. Furthermore, we demonstrate that if the network is perfectly trained, i.e., if it is a well-trained deep inverse network, the refinement step does not alter the results, effectively acting as the identity operator $\mathbf{I}$.

\begin{theorem}
\label{theorem_2}
If the reconstruction network $\hat{f}(\boldsymbol{y})$ is well-trained and contains a deep range-null-space decomposition, then the obtained solution using the projection-based correction is equivalent to the network output, i.e., the correction does not alter the final estimation.
\end{theorem}
\begin{proof}
     The proof is included in  SM, Appendix A. 
\end{proof}

\section{Regularized Projection-Based Correction for Noisy Measurements}
\label{section:Noise}

When noise is present in the measurement model, the standard projection-based correction method no longer guarantees optimal reconstruction. Specifically, the following proposition shows that a well-trained network exhibits a bias influenced by the covariance of the noise.

\setcounter{theorem}{1}
\begin{proposition}
\label{prepo:preposition2}
Let $ \hat{f}: \mathcal{Y} \to \mathcal{X} $ be a well-trained reconstruction network satisfying the deep range-null-space decomposition. Given noisy measurements:
\begin{equation}
    \boldsymbol{y} = \mathbf{A} \boldsymbol{x} + \boldsymbol{n}, \quad \boldsymbol{n} \sim \mathcal{N}(\mathbf{0}, \boldsymbol{\Sigma}),
\end{equation}
the expected reconstruction error is given by:
\begin{equation}
    \mathbb{E} \|\boldsymbol{x} - \hat{f}(\boldsymbol{y})\|_2^2 =  \text{Tr} (\mathbf{A}^\dagger \boldsymbol{\Sigma} (\mathbf{A}^\dagger)^\top).
\end{equation}
\end{proposition}

\begin{proof}
The proof is included in Appendix A.
\end{proof}

The following theorem highlights that strictly enforcing the constraint $\mathbf{A} \boldsymbol{x} = \boldsymbol{y}$ leads to overfitting to noise, thereby degrading reconstruction quality.

To enhance robustness, instead of strictly enforcing the constraint, we propose the following regularized optimization problem, which incorporates prior knowledge of the noise distribution ($\boldsymbol{\Sigma}$ covariance matrix) :
\begin{equation}
    \min_{\boldsymbol{x}} \|\boldsymbol{x} - \hat{f}(\boldsymbol{y})\|_2^2 + \lambda (\mathbf{A} \boldsymbol{x} - \boldsymbol{y})^\top \boldsymbol{\Sigma}^{-1} (\mathbf{A} \boldsymbol{x} - \boldsymbol{y}).
\end{equation}
This optimization problem can be efficiently solved using a weighted projection-based correction method.~\footnote{Iterative methods such as Conjugate Gradient can also be employed to avoid the computational cost of direct matrix inversion.} Specifically, we take the gradient of the objective function:
\begin{equation}
    \nabla_{\boldsymbol{x}} \left( \|\boldsymbol{x} - \hat{f}(\boldsymbol{y})\|_2^2 + \lambda (\mathbf{A} \boldsymbol{x} - \boldsymbol{y})^\top \boldsymbol{\Sigma}^{-1} (\mathbf{A} \boldsymbol{x} - \boldsymbol{y}) \right) = 0.
\end{equation}
Expanding and differentiating:
\begin{equation}
    2(\boldsymbol{x} - \hat{f}(\boldsymbol{y})) + 2 \lambda \mathbf{A}^\top \boldsymbol{\Sigma}^{-1} (\mathbf{A} \boldsymbol{x} - \boldsymbol{y}) = 0.
\end{equation}
Rearranging for $\boldsymbol{x}$:
\begin{equation}
    (\mathbf{I} + \lambda \mathbf{A}^\top \boldsymbol{\Sigma}^{-1} \mathbf{A}) \boldsymbol{x} = \hat{f}(\boldsymbol{y}) + \lambda \mathbf{A}^\top \boldsymbol{\Sigma}^{-1} \boldsymbol{y}.
\end{equation}
Thus, the closed-form solution is:
\begin{equation}
    \boldsymbol{x}^* = (\mathbf{I} + \lambda \mathbf{A}^\top \boldsymbol{\Sigma}^{-1} \mathbf{A})^{-1} (\hat{f}(\boldsymbol{y}) + \lambda \mathbf{A}^\top \boldsymbol{\Sigma}^{-1} \boldsymbol{y}).
    \label{eq:proposed_inversion}
\end{equation}
Interestingly, the term $(\mathbf{I} + \lambda \mathbf{A}^\top \boldsymbol{\Sigma}^{-1} \mathbf{A})^{-1}$ can be precomputed as it does not directly depend on the measurements. Therefore, these two projected terms serve as an optional refinement to ensure consistency in the measurements. If the noise covariance is unknown, we can set $\boldsymbol{\Sigma} = \mathbf{I}$, reducing the method to the well-known Tikhonov Regularization.

A visual representation of the proposed method is shown in Fig.~\ref{fig:gr2r}. The proposed approach works as a non-iterative refinement step that enhances the estimation of a deep inverse network while ensuring measurement consistency.

%% file: sec/4_Simulation_and_Results.tex
\begin{figure*}[t]
    \centering
    \includegraphics[width=1\linewidth]{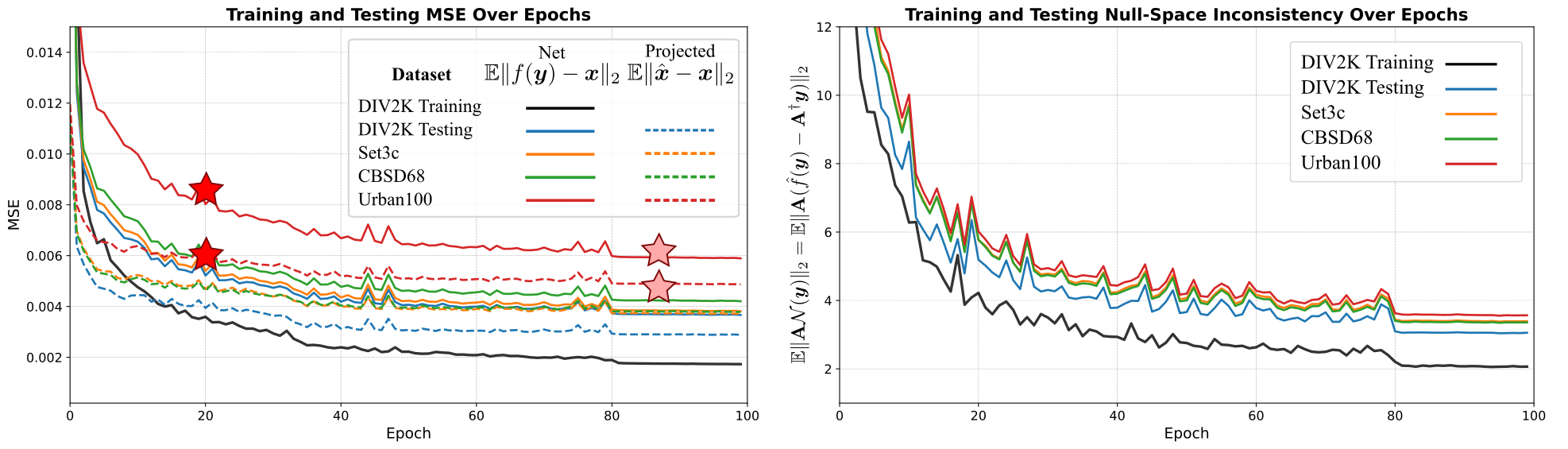}
    \caption{Training and testing comparison of (Left) MSE reconstruction error and (Right) Null-space consistency over epochs. 
    The solid line represents the network performance, while the dashed line corresponds to the proposed Projection-Based Correction.}
    \label{fig:error_curves}
\end{figure*}
\section{Simulations and Results}

In this section, we evaluate the effectiveness of our proposed Projection-Based Correction approach in enhancing deep inverse networks under both noiseless and noisy conditions. We conducted three main experiments to assess its impact. The first experiment examines the role of the Projection-Based Correction step during training by monitoring reconstruction performance over multiple epochs and verifying whether the network meets the well-trained network condition. The second experiment benchmarks the method, evaluating its generalization and robustness across various inverse problems and deep learning architectures. Finally, we assess the performance of the Projection-Based Correction on different testing datasets under varying noise levels by analyzing the hyperparameter $\lambda$. To quantify performance, we use the Peak Signal-to-Noise Ratio (PSNR) and Structural Similarity Index Measure (SSIM), which are standard metrics for image quality.

All simulations were conducted on a workstation equipped with an Intel(R) Xeon(R) W-3223 CPU @ 3.50 GHz processor, 48 GB RAM, and an NVIDIA GeForce RTX 3090 GPU with 24 GB VRAM. The implementation was developed using the DeepInverse~\cite{Tachella_DeepInverse_A_deep_2023} library, and the code is publicly available to ensure reproducibility.

\subsection{Well-Trained Network Condition}

This section analyzes the impact of the proposed Projection-Based Correction on training dynamics, particularly in ensuring that the network inherently satisfies the theoretical conditions of a well-trained inverse model. To evaluate this, we trained a model using the DIV2K dataset, with $256 \times 256$ pixel crops, consisting of 5000 training images and 100 testing images. The selected inverse problem is image deblurring in the noise-free scenario, using a Gaussian blur kernel with $\sigma = [3, 0.15]$. For testing, we evaluated four different datasets: DIV2K testing, Set3c, CBSD68, and Urban100, containing 3, 68, and 100 test images, respectively. Each test image was cropped to $256 \times 256$ for consistency.

For the reconstruction model, we employed a UNet architecture. During training, we monitored the supervised loss $\mathbb{E}\|f(\boldsymbol{y}) - \boldsymbol{x}\|_2^2$ over both the training and testing datasets to track generalization performance. Additionally, we evaluated the projection-based corrected reconstruction $\mathbb{E}\|\hat{\boldsymbol{x}} - \boldsymbol{x}\|_2^2$, where $\hat{\boldsymbol{x}}$ is the output of the proposed Projection-Based Correction method. The results are shown in Figure~\ref{fig:error_curves}. Notably, throughout training, the Projection-Based Correction method consistently outperformed the network’s raw predictions, particularly in the early training stages. As training progresses, the performance gap between the projection and the network output narrows, suggesting that the network gradually learns to reduce the null-space component. However, the network still exhibits better MSE performance on the training set than on the testing set, highlighting the regularization effect enforced by the Projection-Based Correction.
\begin{figure*}
    \centering
    \includegraphics[width=1\linewidth]{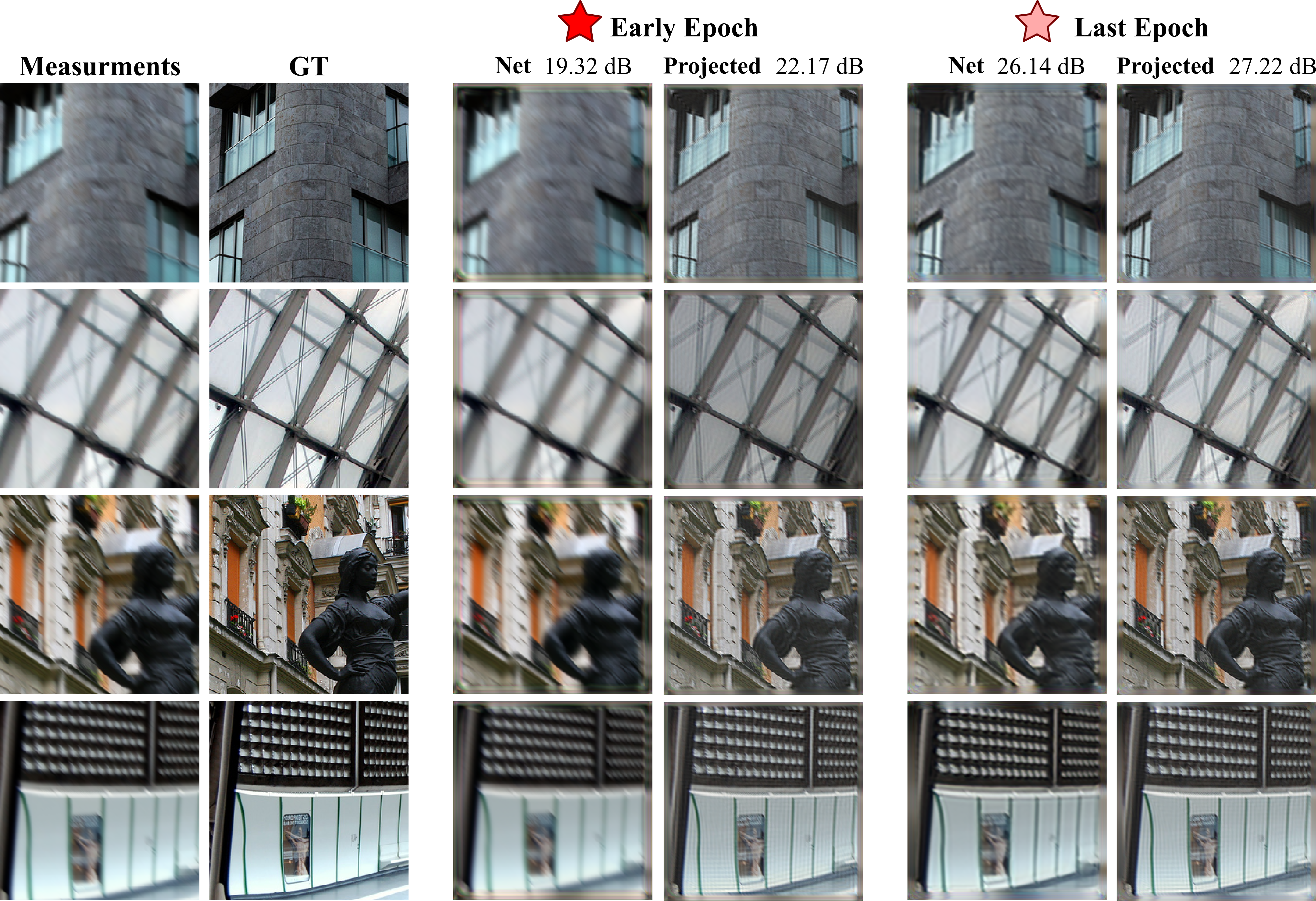}
    \caption{Visual results of the evaluated method. The leftmost column shows the measurement and the ground truth (GT). The results at an early epoch (20 epochs) and the final epoch (100 epochs) are displayed to assess the effect of the Projection-Based Correction method.}
    \label{fig:visual_results}
\end{figure*}

Furthermore, we evaluated whether the network satisfies the well-trained network condition by computing the null-space consistency term:
\begin{equation}
    \mathbb{E}\|\mathbf{A} \mathcal{N}(\boldsymbol{y})\|_2^2 = \mathbb{E}\|\mathbf{A} (\hat{f}(\boldsymbol{y}) - \mathbf{A}^\dagger \boldsymbol{y})\|_2^2.
\end{equation}
This metric quantifies how well the reconstruction remains within the range space of the forward model. The right-hand side of Figure~\ref{fig:error_curves} presents the evolution of this term across different datasets. We observe that the null-space term decreases over training, indicating that the learning process indirectly minimizes this term, supporting the validity of our hypothesis. Interestingly, the null-space error for the training dataset is lower than for the testing datasets, demonstrating that the satisfaction of the well-trained network condition is dataset-dependent. This further highlights the utility of the proposed Projection-Based Correction method in improving consistency across unseen data.

To provide visual validation, we present qualitative results in Figure~\ref{fig:visual_results}. We selected four test images from the Urban100 dataset and compared their reconstructions at an early training stage (epoch 20) and the final epoch (epoch 100). The Projection-Based Correction method significantly improves reconstruction quality at early epochs, leading to PSNR gains of up to 3 dB. Although the performance gap between the network and the projection method diminishes over training, the Projection-Based Correction still provides up to 1 dB improvement in the final reconstructions. These results confirm that enforcing measurement consistency via Projection-Based Correction leads to improved reconstruction fidelity, particularly in early training stages, where the network's implicit null-space regularization is not yet well-formed.

\subsection{Performance of the Proposed Projection-Based Correction for Different Inverse Problems}

This section evaluates the performance of the Projection-Based Correction method across three commonly encountered inverse problems: inpainting, deblurring, and single-pixel imaging (SPI). The inpainting problem considers a random binary mask with a probability of 0.5, where missing pixels must be recovered. The deblurring task involves a Gaussian blur kernel with $\sigma = [3, 0.15]$. The SPI problem is formulated using a measurement matrix with $m = 512$ projections, corresponding to a compression ratio of $0.7\%$. All the networks were trained from scratch, see suplementary material to see the performance of some inverse problem where the weights are share by the authors.

\begin{table}[!t]
\centering
\caption{Performance comparison of different reconstruction methods on the DIV2K test dataset. The best PSNR and SSIM in each row are highlighted in \textbf{bold}.}
\label{tab:SOTA_results}
\begin{spacing}{1.11}
\resizebox{0.8\linewidth}{!}{
\begin{tabular}{llcccc}
\toprule
\textbf{Problem} & \textbf{Network} & \textbf{PSNR} & \textbf{PSNR} & \textbf{SSIM} & \textbf{SSIM} \\
&  & \textbf{Net} & \textbf{Projected} & \textbf{Net} & \textbf{Projected} \\
\midrule
\multirow{4}{*}{SPI} & Unrolled  & 23.62 & \textbf{23.78} & 0.594 & \textbf{0.599} \\
                     & DnCNN    & 24.55 & \textbf{24.73} & 0.645 & \textbf{0.647} \\
                     & Restormer & 25.07 & \textbf{25.27} & 0.651 & \textbf{0.652} \\
                     & DiffUNet & 25.08 & \textbf{25.29} & 0.654 & \textbf{0.654} \\
\midrule
\multirow{4}{*}{Deblurring} 
& Unrolled     & 23.77 & \textbf{25.98} & 0.671 & \textbf{0.799} \\
& DnCNN    & 24.42 & \textbf{25.87} & 0.720 & \textbf{0.819} \\
                            & Restormer & 27.10 & \textbf{27.45} & 0.829 & \textbf{0.842} \\
                            & DiffUNet  & 27.03 & \textbf{27.34} & 0.830 & \textbf{0.841} \\
\midrule
\multirow{4}{*}{Inpainting} & Unrolled  & 27.80 & \textbf{28.54} & 0.671 & \textbf{0.703} \\
                            & DnCNN     & 28.08 & \textbf{29.41} & 0.750 & \textbf{0.816} \\
                            & Restormer & 35.31 & \textbf{35.43} & 0.945 & \textbf{0.946} \\
                            & DiffUNet  & 35.08 & \textbf{35.16} & 0.942 & \textbf{0.943} \\
\bottomrule
\end{tabular}}
\end{spacing}
\end{table}

The evaluation is performed using four deep inverse models: DnCNN~\cite{zhang2017beyond}, a CNN-based denoiser commonly employed in inverse problems; Restormer~\cite{zamir2022restormer}, a transformer-based image restoration model; DiffUnet~\cite{choi2021ilvr}, a diffusion-driven U-Net architecture; and Unrolled Optimization Networks~\cite{monga2021algorithm}, which incorporate iterative optimization steps into a deep learning framework with a Unet as prior. The training and evaluation dataset used for these experiments is DIV2K.

Table~\ref{tab:SOTA_results} presents the results of this evaluation. The performance of each inverse problem varies, with SPI demonstrating the lowest reconstruction quality and inpainting achieving the highest PSNR and SSIM values. The results consistently show that the proposed Projection-Based Correction method improves reconstruction quality across all models and inverse problems. This improvement is attributed to the enforcement of measurement consistency, which refines the learned solutions and mitigates inconsistencies introduced by deep inverse models. 

\begin{figure}[!t]
	\centering
	\includegraphics[width=0.8\linewidth]{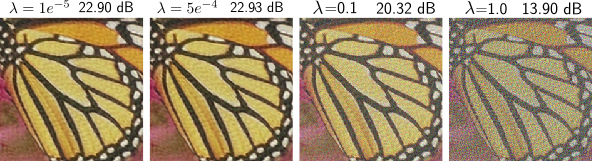}
	\caption{Effect of $\lambda$ parameter in the Projected-Based Correction Method in the present of $0.1$ of Gaussian Noise. }
	\label{fig:resutls_noise_ablaton}
\end{figure}

\begin{table}[!t]
	\centering
	
	\caption{Comparison of reconstruction performance before and after applying the proposed projection-based correction across different noise levels and datasets. }
	\label{tab:noise_results}
	\begin{spacing}{0.8}
		\resizebox{0.73\linewidth}{!}{
			\begin{tabular}{ccccccc}
				\toprule
				\textbf{Noise Level} & \textbf{Dataset} & \textbf{Best } $\lambda$ & \textbf{Method} & \textbf{PSNR} $\uparrow$ & \textbf{SSIM} $\uparrow$ \\
				\midrule
				\multirow{8}{*}{$\sigma = 0.01$}  
				& DIV2K Test & -        & Network  & 28.22 & 0.770 \\
				& DIV2K Test & 0.001    & Projected & \textbf{28.67} & \textbf{0.782} \\
				& Set3C      & -        & Network  & 23.68 & 0.798 \\
				& Set3C      & 0.010    & Projected & \textbf{24.87} & \textbf{0.816} \\
				& CBSD68     & -        & Network  & 25.04 & 0.715 \\
				& CBSD68     & 0.005    & Projected & \textbf{25.52} & \textbf{0.727} \\
				& Urban100   & -        & Network  & 22.88 & 0.653 \\
				& Urban100   & 0.005    & Projected & \textbf{23.69} & \textbf{0.677} \\
				\midrule
				\multirow{8}{*}{$\sigma = 0.05$}  
				& DIV2K Test & -        & Network  & 25.60 & 0.620 \\
				& DIV2K Test & 0.0005   & Projected & \textbf{25.61} & \textbf{0.623} \\
				& Set3C      & -        & Network  & 22.19 & 0.707 \\
				& Set3C      & 0.0005   & Projected & \textbf{23.32} & \textbf{0.718} \\
				& CBSD68     & -        & Network  & 23.88 & 0.604 \\
				& CBSD68     & 0.0001   & Projected & \textbf{23.98} & \textbf{0.614} \\
				& Urban100   & -        & Network  & 21.74 & 0.554 \\
				& Urban100   & 0.001    & Projected & \textbf{21.81} & \textbf{0.556} \\
				\midrule
				\multirow{8}{*}{$\sigma = 0.10$}  
				& DIV2K Test & -        & Network  & \underline{25.88} & \underline{0.643} \\
				& DIV2K Test & 0        & Projected & \underline{25.88} & \underline{0.643} \\
				& Set3C      & -        & Network  & 22.19 & 0.713 \\
				& Set3C      & 0.0005   & Projected & \textbf{22.92} & \textbf{0.753} \\
				& CBSD68     & -        & Network  & \underline{23.59} & \underline{0.595} \\
				& CBSD68     & 0        & Projected & \underline{23.59} & \underline{0.595} \\
				& Urban100   & -        & Network  & 21.71 & 0.545 \\
				& Urban100   & 0.0001   & Projected & \textbf{21.77} & \textbf{0.555} \\
				\midrule
				\multirow{8}{*}{$\sigma = 0.20$}  
				& DIV2K Test & -        & Network  & \underline{24.46} & \underline{0.583} \\
				& DIV2K Test & 0        & Projected & \underline{24.46} & \underline{0.583} \\
				& Set3C      & -        & Network  & 20.33 & 0.621 \\
				& Set3C      & 0.001    & Projected & \textbf{20.34} & \underline{0.621} \\
				& CBSD68     & -        & Network  & \underline{22.43} & \underline{0.528} \\
				& CBSD68     & 0        & Projected & \underline{22.43} & \underline{0.528} \\
				& Urban100   & -        & Network  & 20.64 & 0.477 \\
				& Urban100   & 0.0001   & Projected & \underline{20.64} & \underline{0.477} \\
				\midrule
				\multirow{8}{*}{$\sigma = 0.30$}  
				& DIV2K Test & -        & Network  & \underline{23.52} & \underline{0.536} \\
				& DIV2K Test & 0        & Projected & \underline{23.52} & \underline{0.536} \\
				& Set3C      & -        & Network  & 19.45 & 0.574 \\
				& Set3C      & 0.001    & Projected & \textbf{19.46} & \underline{0.574} \\
				& CBSD68     & -        & Network  & \underline{21.72} & \underline{0.484} \\
				& CBSD68     & 0        & Projected & \underline{21.72} & \underline{0.484} \\
				& Urban100   & -        & Network  & \underline{20.05} & \underline{0.435} \\
				& Urban100   & \underline{0.0001}   & Projected & \underline{20.05} & \underline{0.435} \\
				\bottomrule
		\end{tabular}}
	\end{spacing}
\end{table}

The largest performance gain is observed in the deblurring task, where the Unrolled network benefits from an increase of over 2 dB in PSNR.  In contrast, the smallest improvement occurs in the SPI task, where the projection step yields only a 0.2 dB gain. This is likely due to the extremely low compression ratio, which limits the amount of information encoded in the range-space component, thereby reducing the impact of measurement consistency enforcement. 

The improvement also depends on the intrinsic generalization ability of the network. Models such as DiffUnet, which already exhibit strong performance, show smaller performance gaps between their original output and the projected reconstruction. This effect is most pronounced in the inpainting problem, where the difference between the two methods is marginal. This observation is consistent with Theorem~\ref{theorem_2}, which establishes that when a network is well-trained, the projection-based correction converges to the network’s output, thereby reinforcing its correctness rather than altering it significantly. Additional evaluations on alternative datasets, including Set3c, CBSD68, and Urban100, are presented in the Appendix B.

\subsection{Evaluation on Out-of-Distribution Data and Varying Noise Levels}

In this section, we assess the performance of the proposed Projection-Based Correction method on out-of-distribution data across different levels of Gaussian noise. To this end, we trained the DnCNN~\cite{zhang2017beyond} network on the DIV2K dataset using the deblurring inverse problem, where the measurements were corrupted with additive Gaussian noise at varying levels, specifically $\sigma=[0.01, 0.05, 0.1, 0.2, 0.3]$. For each noise level, the network was trained for 200 epochs to ensure convergence.

\begin{figure}[!b]
    \centering
    \includegraphics[width=1\linewidth]{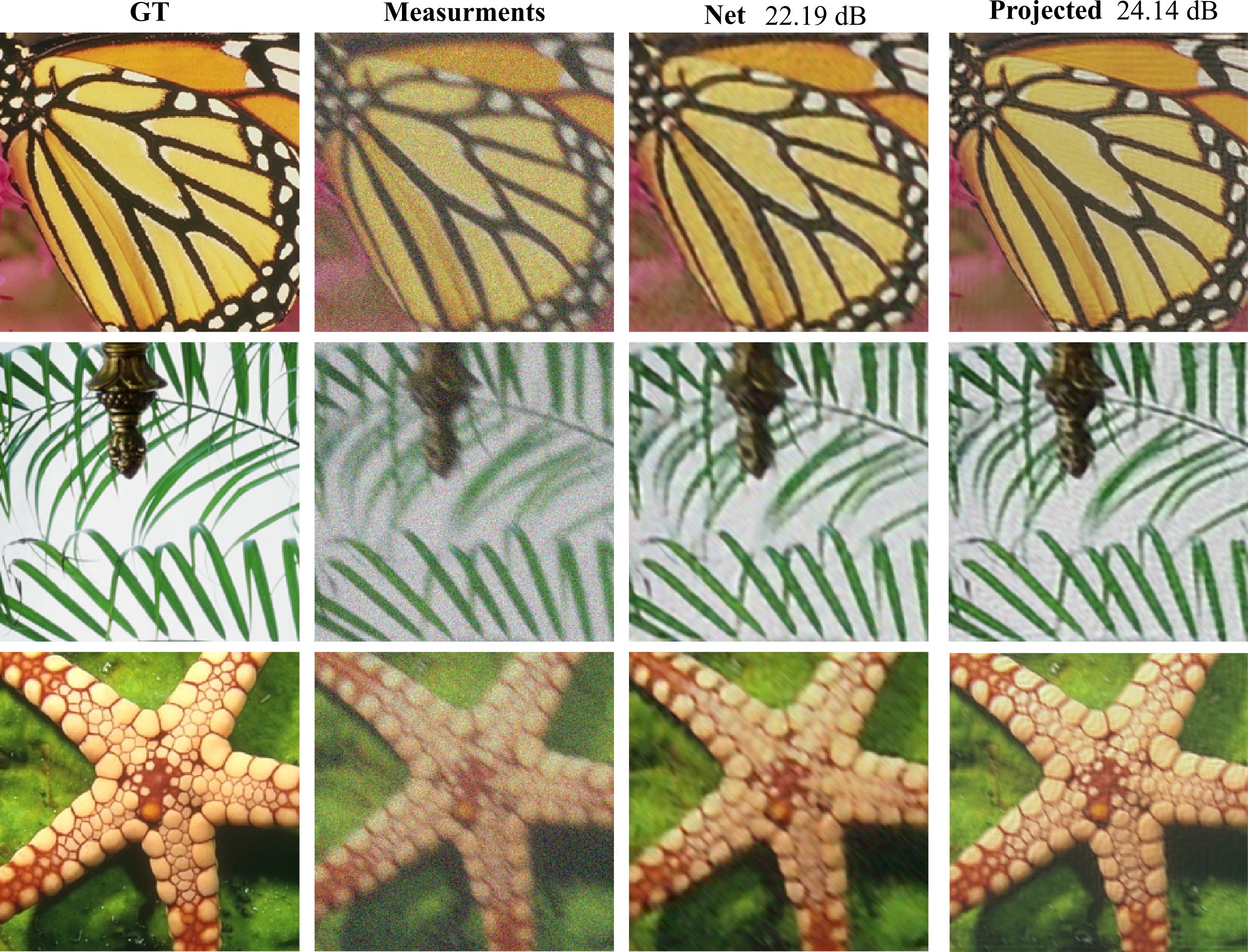}
    \caption{Visual results of the evaluated method. The leftmost column shows the measurement and the ground truth (GT). The results at an early epoch (20 epochs) and the final epoch (100 epochs) are displayed to assess the effect of the Projection-Based Correction method.}
    \label{fig:resutls_noise_final}
\end{figure}

The proposed method was applied using the projection step described in Equation~\eqref{eq:proposed_inversion}. Since the projection introduces a regularization hyperparameter $\lambda$, we performed a grid search to select the optimal value for each noise level. The role of $\lambda$ is to enforce measurement consistency; however, in the presence of noise, an excessively large $\lambda$ can degrade network estimations by amplifying noise rather than recovering structural details. A visual representation for $\sigma = 0.1$ is presented in Figure~\ref{fig:resutls_noise_ablaton} to illustrate this effect. Consequently, The best-performing $\lambda$ was used to generate the final reconstruction results, which are reported in Table~\ref{tab:noise_results}. It provides a quantitative comparison of reconstruction accuracy before and after applying the projection-based correction, reporting PSNR and SSIM metrics across different noise levels and datasets. The results consistently show an improvement, particularly in low-noise scenarios ($\sigma = [0.01, 0.05]$) and for certain datasets in moderate noise conditions ($\sigma = 0.1$). However, as suggested by Proposition \ref{prepo:preposition2}, enforcing strict measurement consistency in high-noise conditions does not necessarily lead to optimal performance. Interestingly, the proposed method exhibits greater performance gains on out-of-distribution datasets, such as Set3C and Urban100, compared to the DIV2K test set, demonstrating its robustness to domain shifts.

To further illustrate the impact of the projection-based correction, Figure~\ref{fig:resutls_noise_final} presents qualitative results for Set3C with noise level $\sigma = 0.1$. The visualizations confirm that the proposed correction method enhances fine details in deblurring tasks by improving alignment with the measurement model. The projection effectively reduces noise-induced artifacts while preserving structural information, reinforcing the importance of enforcing measurement consistency in inverse problems.



\section{Conclusion}
This work introduced a Projection-Based Correction method to enhance deep inverse networks by enforcing strict measurement consistency, addressing a key limitation of deep learning-based solvers that often violate physical constraints. Through extensive evaluations on inpainting, deblurring, and single-pixel imaging, our method consistently improved reconstruction quality across multiple architectures, including DnCNN, Restormer, DiffUnet, and PnP models, particularly in low- and moderate-noise scenarios. While the proposed approach effectively refines network outputs by mitigating null-space artifacts, its performance is limited in high-noise conditions, where enforcing measurement consistency can amplify noise. Future work should explore adaptive regularization strategies to balance measurement consistency and noise robustness.

%% file: sec/X_suppl.tex

\section*{Appendix A}

This section proof the Propositions and Theorem of the proposed method.

\subsection*{Proof of Proposition 1}
\begin{proof}
Inspiring by ~\cite{chen2021equivariant} we have that, if $ \hat{f} $ is well-trained, it satisfies the decomposition:
\begin{equation}
    \hat{f}(\boldsymbol{y}) = \mathbf{A}^\dagger \boldsymbol{y} + \mathcal{N}(\boldsymbol{y}).
\end{equation}
Using the measurement model $ \boldsymbol{y} = \mathbf{A} \boldsymbol{x} $, we substitute:
\begin{equation}
    \hat{f}(\boldsymbol{y}) = \mathbf{A}^\dagger \mathbf{A} \boldsymbol{x} + \mathcal{N}(\boldsymbol{y}).
\end{equation}
Since by definition of a well-trained network
\begin{equation}
    \mathcal{N}(\boldsymbol{y}) = (\mathbf{I} - \mathbf{A}^\dagger \mathbf{A}) \boldsymbol{x},
\end{equation}
we obtain:
\begin{equation}
    \hat{f}(\boldsymbol{y}) = \mathbf{A}^\dagger \mathbf{A} \boldsymbol{x} + (\mathbf{I} - \mathbf{A}^\dagger \mathbf{A}) \boldsymbol{x} = \boldsymbol{x}.
\end{equation}
Thus, the reconstruction is exact, implying that:
\begin{equation}
    \mathbb{E}_{\boldsymbol{x},\boldsymbol{y}} \| \hat{f}(\boldsymbol{y}) - \boldsymbol{x} \|_2^2 = \mathbb{E}_{\boldsymbol{x},\boldsymbol{y}} \| \boldsymbol{x} - \boldsymbol{x} \|_2^2 = 0.
\end{equation}
Since $ \hat{f} $ achieves the minimum possible reconstruction error, it must be a minimizer of Eq.~\eqref{eq:training}.
\end{proof}

\subsection*{Proof of Theorem 1}
\begin{proof}
The given problem is a constrained least squares problem, where we seek the point $\boldsymbol{x}^*$ closest to $\hat{f}(\boldsymbol{y})$ while satisfying the constraint $\mathbf{A}\boldsymbol{x} = \boldsymbol{y}$.

Define the Lagrangian of the problem:
\begin{equation}
    \mathcal{L}(\boldsymbol{x}, \boldsymbol{\lambda}) = \|\boldsymbol{x} - \hat{f}(\boldsymbol{y})\|_2^2 + \boldsymbol{\lambda}^\top (\mathbf{A} \boldsymbol{x} - \boldsymbol{y}),
\end{equation}
where $\boldsymbol{\lambda} \in \mathbb{R}^m$ is the vector of Lagrange multipliers enforcing the constraint.

Taking the derivative with respect to $\boldsymbol{x}$ and setting it to zero:
\begin{equation}
    \frac{\partial \mathcal{L}}{\partial \boldsymbol{x}} = 2(\boldsymbol{x} - \hat{f}(\boldsymbol{y})) + \mathbf{A}^\top \boldsymbol{\lambda} = 0.
\end{equation}
Solving for $\boldsymbol{x}$:
\begin{equation}
    \boldsymbol{x} = \hat{f}(\boldsymbol{y}) - \frac{1}{2} \mathbf{A}^\top \boldsymbol{\lambda}.
\end{equation}

sing the constraint $\mathbf{A} \boldsymbol{x} = \boldsymbol{y}$:
\begin{equation}
    \mathbf{A} \hat{f}(\boldsymbol{y}) - \frac{1}{2} \mathbf{A} \mathbf{A}^\top \boldsymbol{\lambda} = \boldsymbol{y}.
\end{equation}
Multiplying both sides by $(\mathbf{A} \mathbf{A}^\top)^{-1}$:
\begin{equation}
    \boldsymbol{\lambda} = 2 (\mathbf{A} \mathbf{A}^\top)^{-1} (\mathbf{A} \hat{f}(\boldsymbol{y}) - \boldsymbol{y}).
\end{equation}
Substituting $\boldsymbol{\lambda}$ back into the equation for $\boldsymbol{x}$:
\begin{equation}
    \boldsymbol{x}^* = \hat{f}(\boldsymbol{y}) - \mathbf{A}^\top (\mathbf{A} \mathbf{A}^\top)^{-1} (\mathbf{A} \hat{f}(\boldsymbol{y}) - \boldsymbol{y}).
\end{equation}
Using the pseudoinverse identity:
\begin{equation}
    \mathbf{A}^\dagger = \mathbf{A}^\top (\mathbf{A} \mathbf{A}^\top)^{-1},
\end{equation}
we obtain:
\begin{equation}
    \boldsymbol{x}^* = \hat{f}(\boldsymbol{y}) - \mathbf{A}^\dagger (\mathbf{A} \hat{f}(\boldsymbol{y}) - \boldsymbol{y}).
\end{equation}
Expanding:
\begin{equation}
    \boldsymbol{x}^* = \mathbf{A}^\dagger \boldsymbol{y} + (\mathbf{I} - \mathbf{A}^\dagger \mathbf{A}) \hat{f}(\boldsymbol{y}).
\end{equation}
Thus, the projected solution satisfies both the constraint and the minimization objective.
\end{proof}

\subsection*{Proof of Theorem 2}
\begin{proof}
Given the decomposition of the reconstruction network:
\begin{equation}
    \hat{f}(\boldsymbol{y}) = \mathbf{A}^\dagger \boldsymbol{y} + \mathcal{N}(\boldsymbol{y}),
\end{equation}
where $\mathcal{N}(\boldsymbol{y})$ belongs to the null space of $A$, i.e., $A \mathcal{N}(\boldsymbol{y}) = 0$. we have that the projection-based correction is given by \begin{equation}
    \boldsymbol{x}^* = \mathbf{A}^\dagger \boldsymbol{y} + (\mathbf{I} - \mathbf{A}^\dagger \mathbf{A}) (\mathbf{A}^\dagger \boldsymbol{y} + \mathcal{N}(\boldsymbol{y})).
\end{equation}
Expanding:
\begin{equation}
    \boldsymbol{x}^* = \mathbf{A}^\dagger \boldsymbol{y} + (\mathbf{I} - \mathbf{A}^\dagger \mathbf{A}) \mathbf{A}^\dagger \boldsymbol{y} + (\mathbf{I} - \mathbf{A}^\dagger \mathbf{A}) \mathcal{N}(\boldsymbol{y}).
\end{equation}
Using the property that $
    (\mathbf{I} - \mathbf{A}^\dagger \mathbf{A}) \mathbf{A}^\dagger  = \mathbf{0}$ we 
 simplify the expression to:
\begin{equation}
    \boldsymbol{x}^* = \mathbf{A}^\dagger \boldsymbol{y} + (\mathbf{I} - \mathbf{A}^\dagger \mathbf{A}) \mathcal{N}(\boldsymbol{y}).
\end{equation}
Since $\mathcal{N}(\boldsymbol{y})$ is in the null space of $\mathbf{A}$, we observe that:
\begin{equation}
    (\mathbf{I} - \mathbf{A}^\dagger\mathbf{A})\mathcal{N}(\boldsymbol{y}) = \mathcal{N}(\boldsymbol{y}),
\end{equation}
which simplifies to:
\begin{equation}
    \boldsymbol{x}^* = \mathbf{A}^\dagger \boldsymbol{y} + \mathcal{N}(\boldsymbol{y}) = \hat{f}(\boldsymbol{y}).
\end{equation}

Thus, the projection-based correction does not alter the network output when the network incorporates a deep range-null-space decomposition.
\end{proof}

\subsection*{Proof of Proposition 2}
\begin{proof}
Since the network $\hat{f}$ satisfies the well-trained decomposition:
\begin{equation}
    \hat{f}(\boldsymbol{y}) = \mathbf{A}^\dagger \boldsymbol{y} + \mathcal{N}(\boldsymbol{y}),
\end{equation}
where $\mathcal{N}(\boldsymbol{y}) = (\mathbf{I} - \mathbf{A}^\dagger \mathbf{A}) \boldsymbol{x}$ represents the null-space component, we substitute the noisy measurement model:
\begin{equation}
    \hat{f}(\boldsymbol{y}) = \mathbf{A}^\dagger (\mathbf{A} \boldsymbol{x} + \boldsymbol{n}) + \mathcal{N}(\boldsymbol{y}).
\end{equation}
Expanding the terms:
\begin{equation}
    \hat{f}(\boldsymbol{y}) = \mathbf{A}^\dagger \mathbf{A} \boldsymbol{x} + \mathbf{A}^\dagger \boldsymbol{n} + (\mathbf{I} - \mathbf{A}^\dagger \mathbf{A}) \boldsymbol{x}.
\end{equation}
Rearranging correctly:
\begin{equation}
    \hat{f}(\boldsymbol{y}) = \boldsymbol{x} + \mathbf{A}^\dagger \boldsymbol{n}.
\end{equation}
Thus, the reconstruction error is:
\begin{equation}
    \boldsymbol{x} - \hat{f}(\boldsymbol{y}) = - \mathbf{A}^\dagger \boldsymbol{n}.
\end{equation}
Taking the expectation of the squared norm:
\begin{equation}
    \mathbb{E} \|\boldsymbol{x} - \hat{f}(\boldsymbol{y})\|_2^2 = \mathbb{E} \| \mathbf{A}^\dagger \boldsymbol{n} \|_2^2.
\end{equation}
From standard properties of Gaussian noise:
\begin{equation}
    \mathbb{E} \| \mathbf{A}^\dagger \boldsymbol{n} \|_2^2 = \text{Tr} (\mathbf{A}^\dagger \boldsymbol{\Sigma} (\mathbf{A}^\dagger)^\top),
\end{equation}
where $\boldsymbol{\Sigma}$ is covariance matrix of the noise. Thus, we obtain 
\begin{equation}
    \mathbb{E} \|\boldsymbol{x} - \hat{f}(\boldsymbol{y})\|_2^2 =  \text{Tr} (\mathbf{A}^\dagger \boldsymbol{\Sigma} (\mathbf{A}^\dagger)^\top).
\end{equation}
\end{proof}

\section*{Appendix B}
\begin{table}[!b]
\renewcommand{\arraystretch}{1.25}
\centering
\caption{Reconstruction Quality Across $10$ KAIST Testing Scenes.}
\resizebox{1\columnwidth}{!}{%
\setlength\tabcolsep{0.5cm}
\begin{tabular}{|c|cc|cc|}
\hline
\multirow{2}{*}{\textbf{Method}} & \multicolumn{2}{c|}{\textbf{PSNR {[}dB{]}}} & \multicolumn{2}{c|}{\textbf{SSIM}} \\ 
\cline{2-5} 
& \multicolumn{1}{c|}{\textbf{DGSMP}} & \textbf{MMU} & \multicolumn{1}{c|}{\textbf{DGSMP}} & \textbf{MMU} \\ 
\hline
\textbf{Baseline}  & \multicolumn{1}{c|}{30.28 $\pm$ 3.09} & 31.85 $\pm$ 2.89  & \multicolumn{1}{c|}{0.92 $\pm$ 0.02} & 0.93 $\pm$ 0.01  \\  
\textbf{Projected} & \multicolumn{1}{c|}{\textbf{33.45 $\pm$ 3.17}} & \textbf{32.17 $\pm$ 2.62}  & \multicolumn{1}{c|}{\textbf{0.94 $\pm$ 0.01}} & \textbf{0.95 $\pm$ 0.02}  \\  
\hline
\end{tabular}
}
\label{tab:projected_results}
\end{table}

\begin{table}[!t]
\centering
\caption{Performance comparison of different reconstruction methods across inverse problems and datasets. The best PSNR and SSIM in each row are highlighted in \textbf{bold}.}
\label{tab:results_more}
\begin{spacing}{0.8}
\resizebox{\linewidth}{!}{
\begin{tabular}{lllcccc}
\toprule
\textbf{Problem} & \textbf{Dataset} & \textbf{Network} & \textbf{PSNR} & \textbf{PSNR} & \textbf{SSIM} & \textbf{SSIM} \\
&  &  & \textbf{Net} & \textbf{Projected} & \textbf{Net} & \textbf{Projected} \\
\midrule
\multirow{16}{*}{SPI}  
                     & \multirow{4}{*}{Set3c}    & Unrolled  & 17.72 & \textbf{17.77} & 0.480 & \textbf{0.480} \\
                     &                           & DnCNN     & 18.13 & \textbf{18.23} & 0.544 & \textbf{0.540} \\
                     &                           & Restormer & 18.72 & \textbf{18.88} & 0.591 & \textbf{0.586} \\
                     &                           & DiffUNet  & 18.78 & \textbf{18.95} & 0.596 & \textbf{0.592} \\
\cmidrule{2-7}
                     & \multirow{4}{*}{CBSD68}  & Unrolled  & 21.43 & \textbf{21.49} & 0.531 & \textbf{0.532} \\
                     &                           & DnCNN     & 21.89 & \textbf{21.99} & 0.566 & \textbf{0.567} \\
                     &                           & Restormer & 21.78 & \textbf{21.93} & 0.549 & \textbf{0.550} \\
                     &                           & DiffUNet  & 21.87 & \textbf{22.03} & 0.553 & \textbf{0.554} \\
\cmidrule{2-7}
                     & \multirow{4}{*}{Urban100} & Unrolled  & 20.18 & \textbf{20.27} & 0.507 & \textbf{0.512} \\
                     &                           & DnCNN     & 20.65 & \textbf{20.75} & 0.559 & \textbf{0.558} \\
                     &                           & Restormer & 21.29 & \textbf{21.44} & 0.597 & \textbf{0.597} \\
                     &                           & DiffUNet  & 21.38 & \textbf{21.54} & 0.601 & \textbf{0.601} \\
\midrule
\multirow{12}{*}{Deblurring} 
                     & \multirow{4}{*}{Set3c}    & Unrolled  & 23.77 & \textbf{25.98} & 0.671 & \textbf{0.799} \\
                     &                           & DnCNN     & 24.42 & \textbf{25.87} & 0.720 & \textbf{0.819} \\
                     &                           & Restormer & 27.10 & \textbf{27.45} & 0.829 & \textbf{0.842} \\
                     &                           & DiffUNet  & 27.03 & \textbf{27.34} & 0.830 & \textbf{0.841} \\
\cmidrule{2-7}
                     & \multirow{4}{*}{CBSD68}  & Unrolled  & 22.39 & \textbf{22.81} & 0.832 & \textbf{0.841} \\
                     &                           & DnCNN     & 23.07 & \textbf{24.56} & 0.667 & \textbf{0.775} \\
                     &                           & Restormer & 24.71 & \textbf{24.99} & 0.775 & \textbf{0.788} \\
                     &                           & DiffUNet  & 24.71 & \textbf{24.95} & 0.775 & \textbf{0.786} \\
\cmidrule{2-7}
                     & \multirow{4}{*}{Urban100} & Unrolled  & 22.92 & \textbf{23.24} & 0.757 & \textbf{0.774} \\
                     &                           & DnCNN     & 20.73 & \textbf{22.21} & 0.603 & \textbf{0.723} \\
                     &                           & Restormer & 27.10 & \textbf{27.45} & 0.829 & \textbf{0.842} \\
                     &                           & DiffUNet  & 22.81 & \textbf{23.08} & 0.754 & \textbf{0.769} \\
\midrule
\multirow{12}{*}{Inpainting} 
                     & \multirow{4}{*}{Set3c}    & Unrolled  & 27.80 & \textbf{28.54} & 0.671 & \textbf{0.703} \\
                     &                           & DnCNN     & 28.08 & \textbf{29.41} & 0.750 & \textbf{0.816} \\
                     &                           & Restormer & 35.31 & \textbf{35.43} & 0.945 & \textbf{0.946} \\
                     &                           & DiffUNet  & 35.08 & \textbf{35.16} & 0.942 & \textbf{0.943} \\
\cmidrule{2-7}
                     & \multirow{4}{*}{CBSD68}  & Unrolled  & 31.43 & \textbf{31.46} & 0.964 & \textbf{0.964} \\
                     &                           & DnCNN     & 26.52 & \textbf{27.53} & 0.772 & \textbf{0.827} \\
                     &                           & Restormer & 30.34 & \textbf{30.38} & 0.916 & \textbf{0.917} \\
                     &                           & DiffUNet  & 30.07 & \textbf{30.09} & 0.910 & \textbf{0.911} \\
\cmidrule{2-7}
                     & \multirow{4}{*}{Urban100} & Unrolled  & 30.64 & \textbf{30.79} & 0.930 & \textbf{0.933} \\
                     &                           & DnCNN     & 24.32 & \textbf{25.50} & 0.731 & \textbf{0.805} \\
                     &                           & Restormer & 30.99 & \textbf{31.16} & 0.934 & \textbf{0.937} \\
                     &                           & DiffUNet  & 30.64 & \textbf{30.79} & 0.930 & \textbf{0.933} \\
\bottomrule
\end{tabular}}
\end{spacing}
\end{table}

\label{Appendix_B}

This section shows additional results for different testing datasets of experiment 6.2 of the main document, summarized in Table~\ref{tab:results_more}.

\subsection{Compressive Spectral Imaging}

We selected two pre-trained networks provided by the original authors for the Compressive Spectral Imaging inverse problem. Specifically, we utilized Deep Gaussian Scale Mixture Prior~\cite{huang2021deep}, and the \textbf{Mask Modeling Uncertainty (MMU)}~\cite{wang2022modeling}. The Kaist dataset was used for testing. The results of the proposed Projected Method are summarized in the Table~\ref{tab:projected_results}, where we validate its effectiveness against state-of-the-art approaches that were not trained from scratch. Notably, our method consistently achieves competitive performance, highlighting its robustness and adaptability across different reconstruction scenarios.

%% file: JPCSLaTeXGuidelines.bbl
\providecommand{\newblock}{}
\begin{thebibliography}{10}
\expandafter\ifx\csname url\endcsname\relax
  \def\url#1{{\tt #1}}\fi
\expandafter\ifx\csname urlprefix\endcsname\relax\def\urlprefix{URL }\fi
\providecommand{\eprint}[2][]{\url{#2}}

\bibitem{ongie2020deep}
Ongie G, Jalal A, Metzler C~A, Baraniuk R~G, Dimakis A~G and Willett R 2020
  {\em IEEE Journal on Selected Areas in Information Theory\/} {\bf 1} 39--56

\bibitem{bacca2023computational}
Bacca J, Martinez E and Arguello H 2023 {\em JOSA A\/} {\bf 40} C115--C125

\bibitem{ronneberger2015u}
Ronneberger O, Fischer P and Brox T 2015 U-net: Convolutional networks for
  biomedical image segmentation {\em Medical image computing and
  computer-assisted intervention--MICCAI 2015: 18th international conference,
  Munich, Germany, October 5-9, 2015, proceedings, part III 18\/} (Springer) pp
  234--241

\bibitem{zhang2017beyond}
Zhang K, Zuo W, Chen Y, Meng D and Zhang L 2017 {\em IEEE transactions on image
  processing\/} {\bf 26} 3142--3155

\bibitem{zhang2021plug}
Zhang K, Li Y, Zuo W, Zhang L, Van~Gool L and Timofte R 2021 {\em IEEE
  Transactions on Pattern Analysis and Machine Intelligence\/} {\bf 44}
  6360--6376

\bibitem{liang2021swinir}
Liang J, Cao J, Sun G, Zhang K, Van~Gool L and Timofte R 2021 Swinir: Image
  restoration using swin transformer {\em Proceedings of the IEEE/CVF
  international conference on computer vision\/} pp 1833--1844

\bibitem{zamir2022restormer}
Zamir S~W, Arora A, Khan S, Hayat M, Khan F~S and Yang M~H 2022 Restormer:
  Efficient transformer for high-resolution image restoration {\em Proceedings
  of the IEEE/CVF conference on computer vision and pattern recognition\/} pp
  5728--5739

\bibitem{choi2021ilvr}
Choi J, Kim S, Jeong Y, Gwon Y and Yoon S 2021 {\em arXiv preprint
  arXiv:2108.02938\/}

\bibitem{chan2016plug}
Chan S~H, Wang X and Elgendy O~A 2016 {\em IEEE Transactions on Computational
  Imaging\/} {\bf 3} 84--98

\bibitem{romano2017little}
Romano Y, Elad M and Milanfar P 2017 {\em SIAM Journal on Imaging Sciences\/}
  {\bf 10} 1804--1844

\bibitem{monga2021algorithm}
Monga V, Li Y and Eldar Y~C 2021 {\em IEEE Signal Processing Magazine\/} {\bf
  38} 18--44

\bibitem{arridge2019solving}
Arridge S, Maass P, {\"O}ktem O and Sch{\"o}nlieb C~B 2019 {\em Acta
  Numerica\/} {\bf 28} 1--174

\bibitem{adler2018learned}
Adler J and {\"O}ktem O 2018 {\em IEEE transactions on medical imaging\/} {\bf
  37} 1322--1332

\bibitem{shah2018solving}
Shah V and Hegde C 2018 Solving linear inverse problems using gan priors: An
  algorithm with provable guarantees {\em 2018 IEEE international conference on
  acoustics, speech and signal processing (ICASSP)\/} (IEEE) pp 4609--4613

\bibitem{chen2020deep}
Chen D and Davies M~E 2020 Deep decomposition learning for inverse imaging
  problems {\em Computer Vision--ECCV 2020: 16th European Conference, Glasgow,
  UK, August 23--28, 2020, Proceedings, Part XXVIII 16\/} (Springer) pp
  510--526

\bibitem{sonderby2016amortised}
S{\o}nderby C~K, Caballero J, Theis L, Shi W and Husz{\'a}r F 2016 {\em arXiv
  preprint arXiv:1610.04490\/}

\bibitem{schwab2019deep}
Schwab J, Antholzer S and Haltmeier M 2019 {\em Inverse Problems\/} {\bf 35}
  025008

\bibitem{goppel2023data}
G{\"o}ppel S, Frikel J and Haltmeier M 2023 {\em arXiv preprint
  arXiv:2309.06573\/}

\bibitem{angermann2023uncertainty}
Angermann C, G{\"o}ppel S and Haltmeier M 2023 {\em arXiv preprint
  arXiv:2304.06955\/}

\bibitem{wang2022zero}
Wang Y, Yu J and Zhang J 2022 {\em arXiv preprint arXiv:2212.00490\/}

\bibitem{hanke1997regularizing}
Hanke M 1997 {\em Inverse problems\/} {\bf 13} 79

\bibitem{gupta2018cnn}
Gupta H, Jin K~H, Nguyen H~Q, McCann M~T and Unser M 2018 {\em IEEE
  transactions on medical imaging\/} {\bf 37} 1440--1453

\bibitem{raj2019gan}
Raj A, Li Y and Bresler Y 2019 Gan-based projector for faster recovery with
  convergence guarantees in linear inverse problems {\em Proceedings of the
  IEEE/CVF international conference on computer vision\/} pp 5602--5611

\bibitem{rick2017one}
Rick~Chang J, Li C~L, Poczos B, Vijaya~Kumar B and Sankaranarayanan A~C 2017
  One network to solve them all--solving linear inverse problems using deep
  projection models {\em Proceedings of the IEEE International Conference on
  Computer Vision\/} pp 5888--5897

\bibitem{mardani2018neural}
Mardani M, Sun Q, Donoho D, Papyan V, Monajemi H, Vasanawala S and Pauly J 2018
  {\em Advances in Neural Information Processing Systems\/} {\bf 31}

\bibitem{daras2024survey}
Daras G, Chung H, Lai C~H, Mitsufuji Y, Ye J~C, Milanfar P, Dimakis A~G and
  Delbracio M 2024 {\em arXiv preprint arXiv:2410.00083\/}

\bibitem{Tachella_DeepInverse_A_deep_2023}
Tachella J, Chen D, Hurault S, Terris M and Wang A 2023 {DeepInverse: A deep
  learning framework for inverse problems in imaging}
  \urlprefix\url{https://github.com/deepinv/deepinv}

\bibitem{chen2021equivariant}
Chen D, Tachella J and Davies M~E 2021 Equivariant imaging: Learning beyond the
  range space {\em Proceedings of the IEEE/CVF International Conference on
  Computer Vision\/} pp 4379--4388

\bibitem{huang2021deep}
Huang T, Dong W, Yuan X, Wu J and Shi G 2021 Deep gaussian scale mixture prior
  for spectral compressive imaging {\em Proceedings of the IEEE/CVF Conference
  on Computer Vision and Pattern Recognition\/} pp 16216--16225

\bibitem{wang2022modeling}
Wang J, Zhang Y, Yuan X, Meng Z and Tao Z 2022 Modeling mask uncertainty in
  hyperspectral image reconstruction {\em European Conference on Computer
  Vision\/} (Springer) pp 112--129

\end{thebibliography}
